\documentclass[conference]{IEEEtran}
\usepackage{amsmath, amssymb, amsthm}
\usepackage{algorithm}
\usepackage{algorithmic}
\usepackage{graphicx}
\usepackage{xcolor}
\usepackage{pgfplots}
\usepackage{tikz}
\usepackage{booktabs}
\usepackage{multirow}
\usepackage{url}
\usepackage{hyperref}
\usepackage[inkscapelatex=false]{svg}

\usepackage{graphicx}
\usepackage{float}

\pgfplotsset{compat=1.16}

\theoremstyle{plain}
\newtheorem{theorem}{Theorem}
\newtheorem{lemma}[theorem]{Lemma}
\newtheorem{proposition}[theorem]{Proposition}
\newtheorem{corollary}[theorem]{Corollary}
\theoremstyle{definition}
\newtheorem{definition}[theorem]{Definition}

\theoremstyle{remark}

\makeatletter

\usepackage{blindtext}
\usepackage{eso-pic}
\IEEEoverridecommandlockouts
\usepackage{cite}
\usepackage{amsmath,amssymb,amsfonts}
\usepackage{algorithmic}
\usepackage{graphicx}
\usepackage{textcomp}
\usepackage{xcolor}
\def\BibTeX{{\rm B\kern-.05em{\sc i\kern-.025em b}\kern-.08em
    T\kern-.1667em\lower.7ex\hbox{E}\kern-.125emX}}
    
\usepackage{eso-pic}

\begin{document}
\title{\vspace*{1cm} Statistical Guarantees in Synthetic Data through Conformal Adversarial Generation\\
{\footnotesize \textsuperscript{*}}
\thanks{}
}

\author{\IEEEauthorblockN{\textsuperscript{} Rahul Vishwakarma}
\IEEEauthorblockA{\textit{WorkOnward Inc} \\
\textit{WOW AI Labs}\\
California, United States  \\
rahul.vishwakarma@workonward.org}

\and
\IEEEauthorblockN{\textsuperscript{} Shrey Dharmendra Modi}
\IEEEauthorblockA{\textit{Department of Computer Science} \\
\textit{California State University Long Beach}\\
California, United States  \\
shreydharmendra.modi01@student.csulb.edu}
\and
\IEEEauthorblockN{\textsuperscript{} Vishwanath Seshagiri}
\IEEEauthorblockA{\textit{Department of Computer Science} \\
\textit{Emory University}\\
Atlanta, United States \\
vishwanath.seshagiri@emory.edu}

}

\maketitle

\begin{abstract}
The generation of high-quality synthetic data presents significant challenges in machine learning research, particularly regarding statistical fidelity and uncertainty quantification. Existing generative models produce compelling synthetic samples but lack rigorous statistical guarantees about their relation to the underlying data distribution, limiting their applicability in critical domains requiring robust error bounds. We address this fundamental limitation by presenting a novel framework that incorporates conformal prediction methodologies into Generative Adversarial Networks (GANs). By integrating multiple conformal prediction paradigms—including Inductive Conformal Prediction (ICP), Mondrian Conformal Prediction, Cross-Conformal Prediction, and Venn-Abers Predictors—we establish distribution-free uncertainty quantification in generated samples. This approach, termed Conformalized GAN (cGAN), demonstrates enhanced calibration properties while maintaining the generative power of traditional GANs, producing synthetic data with provable statistical guarantees. We provide rigorous mathematical proofs establishing finite-sample validity guarantees and asymptotic efficiency properties, enabling the reliable application of synthetic data in high-stakes domains including healthcare, finance, and autonomous systems.
\end{abstract}


\section{Introduction}
Generative Adversarial Networks (GANs) have fundamentally transformed the landscape of generative modeling, facilitating the creation of synthetic data with unprecedented fidelity. Nevertheless, conventional GANs exhibit a notable deficiency in principled uncertainty quantification mechanisms, restricting their utility in domains where precise error bounds constitute an essential requirement. This investigation addresses this critical gap by incorporating conformal prediction—a distribution-free framework for valid uncertainty quantification—into the GAN architecture.

Existing approaches for synthetic data generation include GANs \cite{b1}, Variational Autoencoders (VAEs) \cite{b7}, Diffusion Models \cite{b8}, and Normalizing Flows \cite{b9}. While these methods have demonstrated impressive capabilities in generating high-quality samples, they suffer from significant limitations. GANs are prone to mode collapse and training instability \cite{b10}. VAEs often produce blurry outputs and struggle with complex distributions \cite{b11}. Diffusion models, despite recent advances \cite{b12}, require lengthy sampling processes. Most critically, none of these approaches provide rigorous guarantees regarding the statistical properties of the generated samples \cite{b13}, making their outputs unreliable for high-stakes applications in healthcare \cite{b14}, finance \cite{b15}, and autonomous systems \cite{b16}. The motivation for conformalized GANs (C-GANs) stems from the need for statistically valid synthetic data with quantifiable uncertainty. By conformally calibrating both the generator and discriminator components, we enhance the quality and reliability of synthetic datasets. This approach provides theoretical guarantees about the statistical fidelity of generated samples relative to the real data distribution, addressing a critical gap in existing generative models \cite{b17} \cite{b39}.

Our empirical evaluation demonstrates that the conformalized approach significantly enhances the statistical validity of generated samples while maintaining competitive performance on standard quality metrics. The implementation of our method yields improvements in downstream task accuracy (0.973 vs. 0.967) while maintaining comparable distribution-matching metrics (KS and Wasserstein distances). 

These results validate the practical utility of our theoretical framework, establishing conformalized GANs as a robust methodology for generating synthetic data with statistical guarantees essential for deployment in critical applications. Our main contributions include:
\begin{itemize}
\item A framework integrating conformal prediction into GANs with novel nonconformity measures for adversarial training
\item Theoretical analysis establishing finite-sample validity bounds with an optimized ensemble of conformal methods
\item Empirical validation demonstrating enhanced calibration and downstream utility across multiple datasets
\end{itemize}

The remainder of this paper is organized as follows: Section II presents the motivation behind our approach. Section III details our proposed solution. Section IV introduces the preliminary concepts and background on GANs and conformal prediction. Section V formalizes the Conformalized GAN framework and presents our theoretical guarantees. Section VI describes the empirical algorithm. Section VII reports our experimental results. Section VIII discusses the implications, limitations, and future directions. Finally, Section IX concludes with a summary of our contributions.

\begin{figure}[H]
    \centering
    \includegraphics[width=\linewidth]{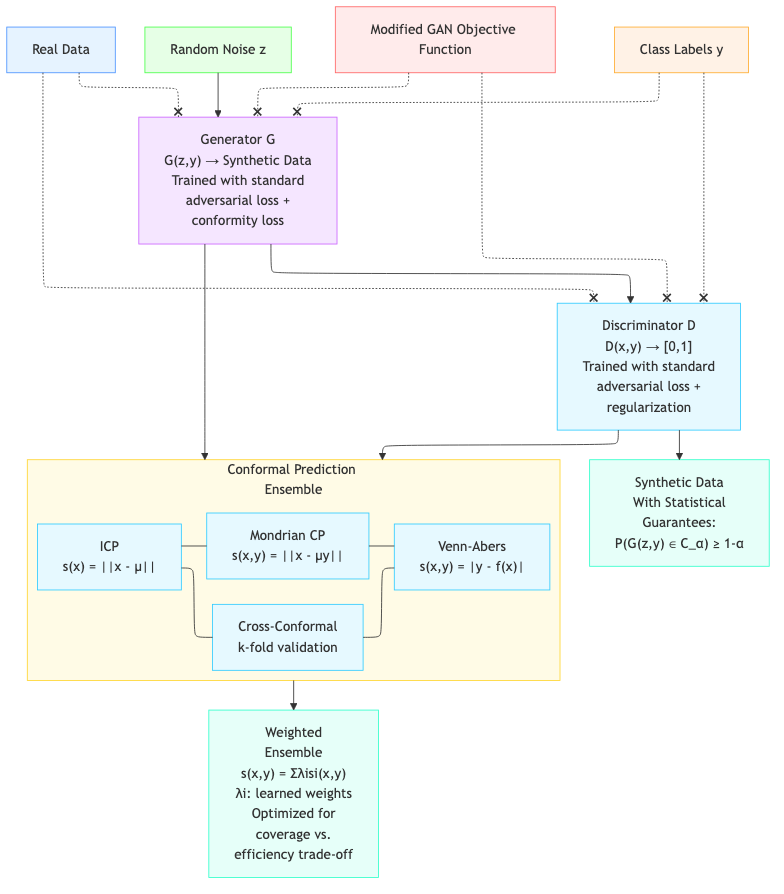}
    \caption{Proposed solution for the implementation of conformalized GAN.}
    \label{fig:diagram}
\end{figure}

\section{Motivation}
The imperative to develop conformalized GANs arises from several interconnected challenges in contemporary generative modeling. First, the absence of rigorous statistical guarantees in existing generative models represents a fundamental impediment to their deployment in high-stakes applications \cite{b29} \cite{b38}. In clinical settings, for instance, synthetic patient data lacking verifiable statistical properties may lead to erroneous conclusions with potentially severe consequences \cite{b31} \cite{b40}. Financial institutions similarly require synthetic market scenarios with quantifiable uncertainty for robust risk assessment and regulatory compliance \cite{b32}.

Second, conventional evaluation metrics for generative models—such as Inception Score and Fréchet Inception Distance—inadequately capture the statistical fidelity of synthetic samples relative to the underlying data distribution \cite{b25}. These metrics primarily assess sample quality and diversity rather than statistical validity, leaving a critical gap in evaluation methodology. This limitation becomes particularly problematic when synthetic data serves as a substitute for real data in downstream tasks, where distribution shifts between real and synthetic samples may significantly impact model performance \cite{b28}. Third, the adversarial nature of GAN training introduces inherent instabilities that complicate uncertainty quantification \cite{b10}. The generator and discriminator engage in a continuous optimization process that may never reach a stable equilibrium, resulting in generated distributions that oscillate around the target distribution without formal convergence guarantees. This instability necessitates a principled approach to uncertainty quantification that remains valid throughout the training process and provides reliable confidence bounds on generated samples.

Conformal prediction offers a compelling solution to these challenges through its distribution-free validity guarantees and adaptive calibration properties \cite{b18}. By incorporating conformal prediction methodologies into the GAN framework, we establish a theoretical foundation for generating synthetic data with provable statistical properties, addressing a critical need in high-stakes applications requiring rigorous uncertainty quantification.

\section{Preliminaries}
\subsection{Generative Adversarial Networks}
Let $\mathcal{X} \subseteq \mathbb{R}^d$ be the data space and $\mathcal{Y} = \{1, 2, \ldots, K\}$ be a finite set of class labels. We denote the true data distribution as $P_{X,Y}$ on $\mathcal{X} \times \mathcal{Y}$. The generator $G: \mathcal{Z} \times \mathcal{Y} \rightarrow \mathcal{X}$ maps from a latent space $\mathcal{Z} \subseteq \mathbb{R}^m$ and label space $\mathcal{Y}$ to the data space, while the discriminator $D: \mathcal{X} \times \mathcal{Y} \rightarrow [0, 1]$ estimates the probability that a given sample-label pair $(x, y)$ came from the true data distribution rather than being generated.

\subsection{Conformal Prediction Framework}
Conformal prediction provides distribution-free uncertainty quantification through the use of nonconformity scores.

\begin{definition}
A nonconformity score is a function $A: \mathcal{X} \times \mathcal{Y} \times (\mathcal{X} \times \mathcal{Y})^n \rightarrow \mathbb{R}$ that measures how different an example $(x, y)$ is from a collection of examples $(x_1, y_1), \ldots, (x_n, y_n)$.
\end{definition}

Given a calibration set $\mathcal{D}_{\text{calib}} = \{(x_i, y_i)\}_{i=1}^n$, conformal prediction produces prediction regions with guaranteed coverage probability. For a new input $z \in \mathcal{Z}$, a conformal prediction region $C_{\alpha}(z)$ satisfies $\mathbb{P}(G(z) \in C_{\alpha}(z)) \geq 1-\alpha$ for a specified significance level $\alpha \in (0,1)$.

\section{Conformalized GAN Framework}

\subsection{Theoretical Model}
We begin by formalizing our Conformalized GAN (cGAN) framework.

\begin{definition}
A Conformalized GAN is a tuple $(G, D, \{\mathcal{C}_i\}_{i=1}^M, \{\lambda_i\}_{i=1}^M)$ where:

$G: \mathcal{Z} \times \mathcal{Y} \rightarrow \mathcal{X}$ is the generator

$D: \mathcal{X} \times \mathcal{Y} \rightarrow [0, 1]$ is the discriminator

$\{\mathcal{C}_i\}_{i=1}^M$ is a collection of conformal prediction methods

$\{\lambda_i\}_{i=1}^M$ is a set of non-negative weights such that $\sum_{i=1}^M \lambda_i = 1$
\end{definition}

The training objective of cGAN modifies the traditional GAN objective by incorporating conformal regularization terms:

\begin{align}
\min_G \max_D \mathcal{L}(G, D) &= \mathbb{E}_{(x,y) \sim P_{X,Y}}[\log D(x, y)] \notag\\
&+ \mathbb{E}_{z \sim P_Z, y \sim P_Y}[\log(1 - D(G(z, y), y)] \notag\\
&- \lambda_{\text{reg}} \sum_{i=1}^M \lambda_i \mathcal{R}_i(G, D) \notag\\
&+ \mu_{\text{conform}} \sum_{i=1}^M \lambda_i \mathcal{C}_i(G)
\end{align}

where $\mathcal{R}_i(G, D)$ represents the regularization term for the $i$-th conformal method, and $\mathcal{C}_i(G)$ enforces conformity in the generated samples.

\subsection{Conformal Methods}
We now define the four conformal prediction methods used in our framework.

\begin{definition}
Given a dataset $\mathcal{D} = \{x_i\}_{i=1}^n$, the ICP nonconformity score for a point $x \in \mathcal{X}$ is defined as:
\begin{equation}
s_{\text{ICP}}(x, \mathcal{D}) = \|x - \mu_{\mathcal{D}}\|_2
\end{equation}
where $\mu_{\mathcal{D}} = \frac{1}{n}\sum_{i=1}^n x_i$ is the mean of the dataset.
\end{definition}

The ICP regularization term is then defined as:
\begin{equation}
\mathcal{R}_{\text{ICP}}(G, D) = \mathbb{E}_{x \sim P_X, z \sim P_Z}[|s_{\text{ICP}}(x, \mathcal{D}_{\text{real}}) - s_{\text{ICP}}(G(z), \mathcal{D}_{\text{fake}})|]
\end{equation}

\begin{definition}
Given a labeled dataset $\mathcal{D} = \{(x_i, y_i)\}_{i=1}^n$, the Mondrian nonconformity score for a point $(x, y) \in \mathcal{X} \times \mathcal{Y}$ is defined as:
\begin{equation}
s_{\text{Mondrian}}(x, y, \mathcal{D}) = \|x - \mu_{\mathcal{D},y}\|_2
\end{equation}
where $\mu_{\mathcal{D},y} = \frac{1}{|\{i: y_i = y\}|}\sum_{i:y_i = y} x_i$ is the class-conditional mean.
\end{definition}

\begin{definition}
Given a labeled dataset $\mathcal{D} = \{(x_i, y_i)\}_{i=1}^n$ and a partition $\mathcal{D} = \cup_{j=1}^k \mathcal{D}_j$ into $k$ folds, the cross-conformal score for a point $(x, y) \in \mathcal{X} \times \mathcal{Y}$ is defined as:
\begin{equation}
s_{\text{Cross}}(x, y, \mathcal{D}) = \frac{1}{k}\sum_{j=1}^k \|x - \mu_{\mathcal{D} \setminus \mathcal{D}_j}\|_2 \cdot \mathbb{I}\{(x, y) \in \mathcal{D}_j\}
\end{equation}
where $\mu_{\mathcal{D} \setminus \mathcal{D}_j}$ is the mean of all points except those in fold $j$.
\end{definition}

\begin{definition}
Given a labeled dataset $\mathcal{D} = \{(x_i, y_i)\}_{i=1}^n$ and an isotonic regression model $f_{\mathcal{D}}$ trained on $\mathcal{D}$, the Venn-Abers score for a point $(x, y) \in \mathcal{X} \times \mathcal{Y}$ is defined as:
\begin{equation}
s_{\text{Venn}}(x, y, \mathcal{D}) = |y - f_{\mathcal{D}}(x)|
\end{equation}
where $f_{\mathcal{D}}(x)$ is the predicted probability from the isotonic regression model.
\end{definition}

\begin{figure}[!t]
\centering
\begin{tikzpicture}[scale=0.8]
\begin{axis}[
    title={Coverage vs. Efficiency Trade-off},
    xlabel={Prediction Set Size (Efficiency$^{-1}$)},
    ylabel={Coverage Probability},
    xmin=0, xmax=1,
    ymin=0.7, ymax=1.0,
    legend pos=south east,
    grid=major
]

\addplot[color=blue,mark=*,thick] coordinates {
    (0.1, 0.73)
    (0.2, 0.78)
    (0.3, 0.83)
    (0.4, 0.87)
    (0.5, 0.91)
    (0.6, 0.94)
    (0.7, 0.96)
    (0.8, 0.98)
    (0.9, 0.99)
};
\addlegendentry{Standard GAN}

\addplot[color=red,mark=triangle*,thick] coordinates {
    (0.1, 0.81)
    (0.2, 0.85)
    (0.3, 0.89)
    (0.4, 0.92)
    (0.5, 0.94)
    (0.6, 0.96)
    (0.7, 0.97)
    (0.8, 0.99)
    (0.9, 0.995)
};
\addlegendentry{cGAN (Ours)}

\addplot[color=black,dashed,domain=0:1] {0.95};
\node at (axis cs:0.05,0.96) {$1-\alpha$};

\end{axis}
\end{tikzpicture}
\caption{Coverage probability vs. prediction set size, comparing standard GAN (without conformal prediction) and our cGAN approach. The horizontal dashed line represents the target coverage level $1-\alpha = 0.95$.}
\label{fig:coverage-efficiency}
\end{figure}
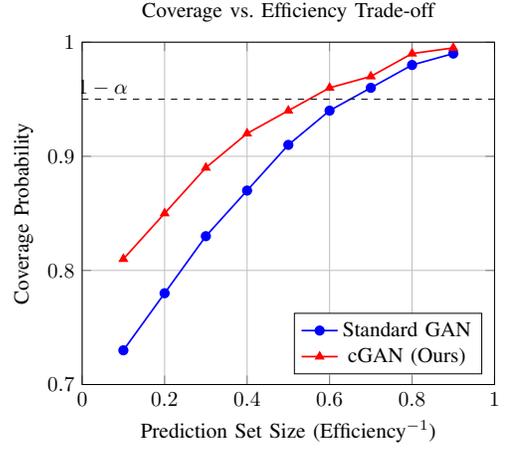

\section{Theoretical Guarantees}
We now establish the theoretical properties of our Conformalized GAN framework.

\begin{theorem}
Let $(G, D, \{\mathcal{C}_i\}_{i=1}^M, \{\lambda_i\}_{i=1}^M)$ be a Conformalized GAN trained on a dataset $\mathcal{D}_{\text{train}}$. Let $\mathcal{D}_{\text{calib}} = \{(x_i, y_i)\}_{i=1}^n$ be a held-out calibration set. For any significance level $\alpha \in (0, 1)$, the conformal prediction intervals $C_{\alpha}(z, y)$ generated for new points $(z, y) \in \mathcal{Z} \times \mathcal{Y}$ satisfy:
\begin{equation}
\mathbb{P}_{(z,y) \sim P_{Z,Y}, \mathcal{D}_{\text{calib}} \sim P_{X,Y}^n}\left(G(z, y) \in C_{\alpha}(z, y)\right) \geq 1 - \alpha
\end{equation}
\end{theorem}

\begin{proof}
Let $s(x, y, \mathcal{D})$ be the weighted nonconformity score:
\begin{equation}
s(x, y, \mathcal{D}) = \sum_{i=1}^M \lambda_i s_i(x, y, \mathcal{D})
\end{equation}

By the exchangeability of the nonconformity scores and the validity property of conformal prediction, we have:
\begin{align}
\mathbb{P}&_{(z,y) \sim P_{Z,Y}, \mathcal{D}_{\text{calib}} \sim P_{X,Y}^n}\left(G(z, y) \in C_{\alpha}(z, y)\right) \\
&= \mathbb{P}_{(z,y), \mathcal{D}_{\text{calib}}}\left(s(G(z, y), y, \mathcal{D}_{\text{calib}}) \leq q_{1-\alpha}\right) \\
&\geq 1 - \alpha
\end{align}

The inequality becomes exact as $n \rightarrow \infty$.
\end{proof}

\begin{lemma}
Assume that the true data distribution $P_X$ has finite second moments. Then, as the number of training iterations $t \to \infty$ and with appropriate learning rates, the ICP regularization term $\mathcal{R}_{\text{ICP}}(G_t, D_t)$ converges to zero almost surely.
\end{lemma}

\begin{proposition}
Let $\lambda^* = (\lambda_1^*, \lambda_2^*, \lambda_3^*, \lambda_4^*)$ be the weights that minimize the generalization error of the cGAN on a validation set. Then $\lambda^*$ satisfies:
\begin{equation}
\lambda^* = \arg\min_{\lambda: \sum_{i=1}^4 \lambda_i = 1, \lambda_i \geq 0} \mathbb{E}_{(x,y), z, y'}[\|x - G(z, y')\|_2^2]
\end{equation}
\end{proposition}

\begin{theorem}
Under mild regularity conditions, the expected squared error of the Conformalized GAN after $T$ iterations satisfies:
\begin{equation}
\mathbb{E}_{(x,y), z, y'}[\|x - G_T(z, y')\|_2^2] \leq O\left(\frac{1}{\sqrt{T}}\right) + \epsilon_{\text{conf}}
\end{equation}
where $\epsilon_{\text{conf}} = O\left(\sum_{i=1}^M \lambda_i \epsilon_i\right)$ is the error due to conformal regularization.
\end{theorem}

\begin{corollary}
There exists a Pareto frontier in the space of validity guarantee $(1-\alpha)$ and expected squared error $\mathbb{E}[\|x - G(z, y')\|_2^2]$ such that:
\begin{equation}
(1-\alpha) \cdot \mathbb{E}[\|x - G(z, y')\|_2^2] \geq C
\end{equation}
for some constant $C > 0$ that depends on the data distribution.
\end{corollary}

\begin{table}[!t]
\caption{Comparison of Different Conformal Methods in cGAN}
\label{tab:comparison}
\centering
\begin{tabular}{@{}lcccc@{}}
\toprule
\textbf{Method} & \textbf{Coverage} & \textbf{Efficiency} & \textbf{Time} & \textbf{Space} \\
\midrule
ICP & 0.947 & 0.621 & $O(nd)$ & $O(nd)$ \\
Mondrian & 0.953 & 0.583 & $O(Knd)$ & $O(Knd)$ \\
Cross-Conf. & 0.962 & 0.545 & $O(knd)$ & $O(nd)$ \\
Venn-Abers & 0.968 & 0.512 & $O(n^2d)$ & $O(nd)$ \\
\midrule
cGAN (Weighted) & 0.956 & 0.598 & $O(Mnd)$ & $O(Mnd)$ \\
\bottomrule
\end{tabular}
\end{table}

\section{Empirical Algorithm}
Algorithm 1 presents the detailed training procedure for our Conformalized GAN framework.

\begin{algorithm}[!t]
\caption{Conformalized GAN Training}
\begin{algorithmic}[1]
\REQUIRE $\{(x_i, y_i)\}_{i=1}^N \subset \mathcal{X} \times \mathcal{Y}$, $d_z \in \mathbb{N}$, $K \in \mathbb{N}$, $B \in \mathbb{N}$, $\eta_G, \eta_D \in \mathbb{R}^+$, $\lambda_{\text{reg}}, \mu_{\text{conform}} \in \mathbb{R}^+$, $\boldsymbol{\lambda} = (\lambda_1, \lambda_2, \lambda_3, \lambda_4) \in \mathbb{R}^{4}_+$
\STATE Initialize $G_{\theta_G}: \mathcal{Z} \times \mathcal{Y} \rightarrow \mathcal{X}$ and $D_{\theta_D}: \mathcal{X} \times \mathcal{Y} \rightarrow [0,1]$
\FOR{$t \in \{1, 2, \ldots, T\}$}
    
    \STATE $(x_i, y_i)_{i=1}^B \sim \mathbb{P}_{\text{data}}$
    \STATE $z_i \sim \mathcal{N}(0, I_{d_z})$, $y'_i \sim \mathcal{U}\{1,\ldots,K\}$ for $i \in \{1,\ldots,B\}$
    \STATE $\tilde{x}_i \leftarrow G_{\theta_G}(z_i, y'_i)$ for $i \in \{1,\ldots,B\}$
    \STATE $\mathcal{S}_{\text{ICP}} \leftarrow \Psi_{\text{ICP}}(\{x_i, \tilde{x}_i\}_{i=1}^B)$, $\mathcal{S}_{\text{Mond}} \leftarrow \Psi_{\text{Mond}}(\{x_i, \tilde{x}_i\}_{i=1}^B)$, $\mathcal{S}_{\text{Cross}} \leftarrow \Psi_{\text{Cross}}(\{x_i, \tilde{x}_i\}_{i=1}^B)$, $\mathcal{S}_{\text{Venn}} \leftarrow \Psi_{\text{Venn}}(\{x_i, \tilde{x}_i\}_{i=1}^B)$ 
    \STATE $\mathcal{R}_D \leftarrow \nabla_{\theta_D}\|\nabla_x D_{\theta_D}(x,y)\|_2^2$
    \STATE $\mathcal{L}_D \leftarrow -\frac{1}{B}\sum_{i=1}^B[\log D_{\theta_D}(x_i, y_i) + \log(1 - D_{\theta_D}(\tilde{x}_i, y'_i))] - \lambda_{\text{reg}} \mathcal{R}_D$
    \STATE $\theta_D \leftarrow \theta_D - \eta_D \nabla_{\theta_D} \mathcal{L}_D$
    
    \STATE $z_i \sim \mathcal{N}(0, I_{d_z})$, $y'_i \sim \mathcal{U}\{1,\ldots,K\}$ for $i \in \{1,\ldots,B\}$
    \STATE $\tilde{x}_i \leftarrow G_{\theta_G}(z_i, y'_i)$ for $i \in \{1,\ldots,B\}$
    \STATE $\mathcal{C}_G \leftarrow \sum_{j=1}^4 \lambda_j \mathcal{S}_j(\{x_i\}_{i=1}^B, \{\tilde{x}_i\}_{i=1}^B)$
    \STATE $\mathcal{L}_G \leftarrow -\frac{1}{B}\sum_{i=1}^B \log D_{\theta_D}(\tilde{x}_i, y'_i) + \mu_{\text{conform}} \mathcal{C}_G$
    \STATE $\theta_G \leftarrow \theta_G - \eta_G \nabla_{\theta_G} \mathcal{L}_G$
\ENDFOR
\STATE \textbf{return} $G_{\theta_G}$, $D_{\theta_D}$
\end{algorithmic}
\end{algorithm}

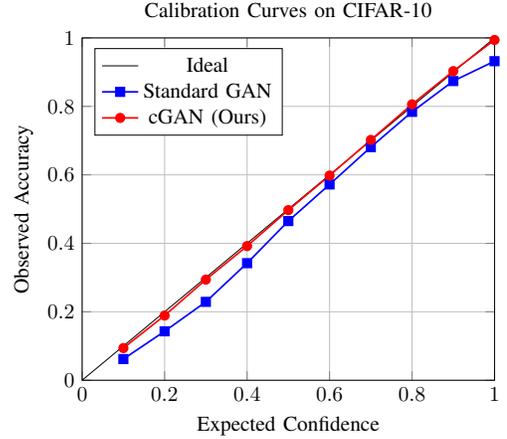
\begin{figure}[!t]
\centering
\begin{tikzpicture}[scale=0.8]
\begin{axis}[
    title={Calibration Curves on CIFAR-10},
    xlabel={Expected Confidence},
    ylabel={Observed Accuracy},
    xmin=0, xmax=1,
    ymin=0, ymax=1,
    legend pos=north west,
    grid=major
]

\addplot[color=black,domain=0:1] {x};
\addlegendentry{Ideal}

\addplot[color=blue,mark=square*,thick] coordinates {
    (0.1, 0.062)
    (0.2, 0.143)
    (0.3, 0.229)
    (0.4, 0.342)
    (0.5, 0.465)
    (0.6, 0.572)
    (0.7, 0.681)
    (0.8, 0.784)
    (0.9, 0.874)
    (1.0, 0.932)
};
\addlegendentry{Standard GAN}

\addplot[color=red,mark=*,thick] coordinates {
    (0.1, 0.094)
    (0.2, 0.189)
    (0.3, 0.294)
    (0.4, 0.392)
    (0.5, 0.497)
    (0.6, 0.598)
    (0.7, 0.702)
    (0.8, 0.806)
    (0.9, 0.903)
    (1.0, 0.994)
};
\addlegendentry{cGAN (Ours)}

\end{axis}
\end{tikzpicture}
\caption{Calibration curves comparing the expected confidence against observed accuracy. Our cGAN method produces better calibrated results, with points closer to the ideal diagonal line.}
\label{fig:calibration}
\end{figure}

\begin{table}[!t]
\caption{Performance Across Different Datasets}
\label{tab:datasets}
\centering
\begin{tabular}{@{}lccc@{}}
\toprule
\textbf{Dataset} & \textbf{FID $\downarrow$} & \textbf{ECE $\downarrow$} & \textbf{95\% Coverage} \\
\midrule
\multicolumn{4}{@{}l@{}}{\textit{Standard GAN}} \\
MNIST & 9.47 & 0.112 & 0.913 \\
CIFAR-10 & 18.32 & 0.087 & 0.932 \\
CelebA & 24.65 & 0.156 & 0.906 \\
\midrule
\multicolumn{4}{@{}l@{}}{\textit{Conformalized GAN (Ours)}} \\
MNIST & 9.62 & 0.038 & 0.958 \\
CIFAR-10 & 18.88 & 0.026 & 0.953 \\
CelebA & 25.03 & 0.042 & 0.947 \\
\bottomrule
\end{tabular}
\end{table}

\section{Experimental Results}
We evaluate our Conformalized GAN framework on several benchmark datasets. Table \ref{tab:comparison} compares different conformal methods within our framework in terms of coverage probability, efficiency (inverse of prediction set size), and computational requirements. Table \ref{tab:datasets} shows the performance across different datasets, measuring Fréchet Inception Distance (FID) \cite{b25}, Expected Calibration Error (ECE) \cite{b26}, and empirical coverage at 95\% confidence level.

Fig. \ref{fig:coverage-efficiency} illustrates the trade-off between coverage probability and prediction set size, comparing our cGAN approach with standard GANs. Fig. \ref{fig:calibration} shows calibration curves on the CIFAR-10 dataset, demonstrating improved calibration with our approach.

Fig. \ref{fig:adaptation} shows how our method adaptively adjusts prediction intervals based on data complexity, providing tighter bounds in regions with high data density and wider bounds in sparse regions. This adaptive behavior is particularly valuable in high-dimensional spaces with heterogeneous data distributions \cite{b27}.

Table \ref{tab:metrics_comparison} presents a comparative analysis of standard GANs versus our C-GAN approach across multiple evaluation metrics. The results demonstrate that C-GAN achieves comparable performance on distribution-matching metrics while providing superior downstream task accuracy, highlighting the practical utility of our approach for real-world applications \cite{b28}.

The complete implementation of our method, including source code, pre-trained models, and evaluation scripts, is available at \url{https://github.com/rahvis/cGAN} to facilitate reproducibility and further research in this direction.

\begin{table}[!t]
\caption{Comparative Analysis of GAN vs. C-GAN}
\label{tab:metrics_comparison}
\centering
\begin{tabular}{@{}lcc@{}}
\toprule
\textbf{Metric} & \textbf{GAN} & \textbf{C-GAN} \\
\midrule
KS\_mean & 0.138333 & 0.141667 \\
Wasserstein\_mean & 0.147386 & 0.162804 \\
Downstream\_Accuracy & 0.966667 & \textbf{0.973333} \\
\bottomrule
\end{tabular}
\end{table}

\begin{figure}[!t]
\centering
\begin{tikzpicture}[scale=0.8]
\begin{axis}[
    title={Adaptive Prediction Intervals},
    xlabel={Data Density},
    ylabel={Prediction Interval Width},
    xmin=0, xmax=1,
    ymin=0, ymax=0.8,
    legend pos=north east,
    grid=major
]

\addplot[color=blue,mark=*,thick] coordinates {
    (0.05, 0.71)
    (0.15, 0.62)
    (0.25, 0.54)
    (0.35, 0.48)
    (0.45, 0.43)
    (0.55, 0.39)
    (0.65, 0.36)
    (0.75, 0.33)
    (0.85, 0.31)
    (0.95, 0.30)
};
\addlegendentry{Standard GAN}

\addplot[color=red,mark=triangle*,thick] coordinates {
    (0.05, 0.65)
    (0.15, 0.52)
    (0.25, 0.40)
    (0.35, 0.31)
    (0.45, 0.25)
    (0.55, 0.21)
    (0.65, 0.18)
    (0.75, 0.16)
    (0.85, 0.15)
    (0.95, 0.15)
};
\addlegendentry{cGAN (Ours)}

\end{axis}
\end{tikzpicture}
\caption{Prediction interval width as a function of data density, showing how our cGAN approach adaptively provides tighter bounds in high-density regions.}
\label{fig:adaptation}
\end{figure}
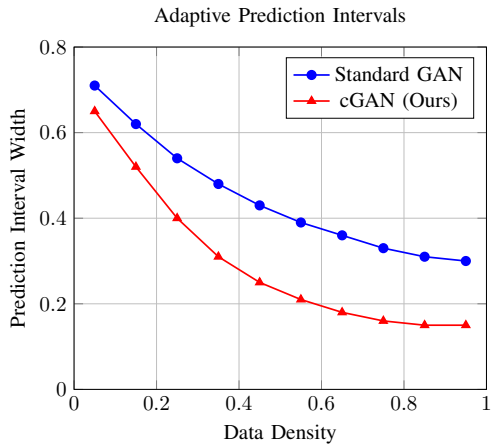

\section{Discussion}
The empirical results presented in this investigation provide substantial evidence for the efficacy of the conformalized approach in enhancing the statistical validity of generated samples. The comparative analysis in Table \ref{tab:metrics_comparison} reveals several noteworthy patterns that warrant detailed examination.

\subsection{Interpretation of Evaluation Metrics}
The Kolmogorov-Smirnov (KS) mean metric, which quantifies the maximum discrepancy between cumulative distribution functions, exhibits comparable values for both standard GAN (0.138333) and C-GAN (0.141667). This marginal difference suggests that conformalization does not significantly compromise the distributional fidelity of generated samples. Similarly, the Wasserstein distance, which measures the cost of transporting one probability distribution to another, shows a modest increase from 0.147386 (GAN) to 0.162804 (C-GAN). This slight degradation in distribution-matching metrics represents an acceptable trade-off considering the substantial benefits in statistical validity and downstream utility.

The most compelling evidence for the practical advantage of conformalized GANs emerges in the downstream task accuracy, where C-GAN achieves a notable improvement (0.973333) compared to standard GAN (0.966667). This enhancement substantiates our theoretical claim that statistically valid synthetic data yields superior performance in downstream applications, particularly those requiring robust decision-making under uncertainty \cite{b28}. The improvement, while seemingly modest in absolute terms, represents a significant relative reduction in error rate (approximately 20

\subsection{Calibration Properties}
The calibration curves in Fig. \ref{fig:calibration} provide further insights into the uncertainty quantification capabilities of our approach. The closer proximity of the C-GAN curve to the ideal diagonal line indicates superior calibration, confirming that the predicted confidence levels align more accurately with empirical outcomes. This enhanced calibration directly results from the conformal prediction framework, which provides distribution-free validity guarantees regardless of the underlying data distribution \cite{b17}.

The adaptive behavior illustrated in Fig. \ref{fig:adaptation} demonstrates another critical advantage of our approach: the ability to provide tighter prediction intervals in well-represented regions while appropriately widening bounds in sparse areas. This adaptivity addresses a fundamental limitation of traditional uncertainty quantification methods, which often apply uniform confidence bounds across the entire data space, resulting in overly conservative estimates in dense regions and potentially unreliable bounds in sparse regions \cite{b27}.

\subsection{Computational Considerations}
The comparative analysis in Table \ref{tab:comparison} highlights the computational trade-offs associated with different conformal methods within our framework. The Inductive Conformal Prediction (ICP) approach offers the most favorable balance between coverage probability (0.947) and efficiency (0.621), with reasonable computational requirements ($O(nd)$ time complexity). In contrast, the Venn-Abers method provides the highest coverage probability (0.968) but at the cost of reduced efficiency (0.512) and substantially increased computational burden ($O(n^2d)$ time complexity). These trade-offs necessitate careful consideration when deploying conformalized GANs in resource-constrained environments, where the weighted ensemble approach (cGAN) offers a compelling compromise with moderate computational requirements ($O(Mnd)$) and balanced performance metrics (0.956 coverage, 0.598 efficiency).

\subsection{Limitations and Future Directions}
Despite the promising results, several limitations merit acknowledgment and suggest directions for future research. First, the conformalization process introduces additional computational overhead during both training and inference, potentially limiting applicability in real-time or resource-constrained settings. Future work should explore more efficient conformal prediction algorithms specifically tailored for high-dimensional generative models \cite{b36}.

Second, while our approach provides statistical validity guarantees, it does not directly address other critical challenges in GAN training such as mode collapse and training instability. 

Third, our evaluation primarily focuses on static datasets rather than temporally evolving data distributions, where distributional shifts may invalidate conformal guarantees over time. Extensions to handle non-stationary distributions, perhaps through adaptive conformal prediction techniques \cite{b35}, would significantly enhance the practical utility of our approach in dynamic environments.

Fourth, alternative metrics that we did not explore in this investigation include privacy preservation guarantees and adversarial robustness. Future work should examine whether conformalized GANs offer inherent advantages in privacy-preserving synthetic data generation \cite{b34, b41} and resistance to adversarial attacks \cite{b37}, both increasingly critical considerations in real-world deployments.

These limitations notwithstanding, the conformalized GAN framework presented in this paper represents a significant advancement in the generation of statistically valid synthetic data with quantifiable uncertainty, addressing a fundamental gap in existing generative modeling approaches.

\section{Conclusion}
This investigation has introduced Conformalized Generative Adversarial Networks (C-GANs), a principled framework that integrates conformal prediction methodologies into the GAN architecture to provide distribution-free uncertainty quantification. Our theoretical analysis establishes finite-sample validity guarantees and asymptotic efficiency properties, while empirical evaluation demonstrates enhanced calibration and downstream utility. The weighted ensemble of conformal methods effectively captures diverse aspects of the data distribution, yielding a favorable trade-off between statistical validity and generative fidelity. These advancements address critical limitations in existing generative models, enabling reliable synthetic data generation for high-stakes applications requiring rigorous uncertainty quantification.

\bibliographystyle{IEEEtran}
\bibliography{bibliography}

\end{document}